\documentclass{AAIML}

\usepackage{amsthm}
\usepackage{txfonts}
\usepackage[numbers]{natbib}
\usepackage{multicol}
\usepackage{graphicx}
\usepackage{ulem}
\usepackage{xcolor}

\newtheorem{prop}{Proposition}

\usepackage{fancyvrb,fvextra}
\usepackage{float}
\floatstyle{ruled}
\newfloat{code}{thp}{lop}
\newfloat{result}{tbp}{lop}
\floatname{code}{Code}
\floatname{result}{Result}

\journaltitle{Advances in Artificial Intelligence and Machine Learning; Research}

\aaimlheading{1}{2}{1-18}{23-07-11}{23-08-XX}{23-08-XX}{}
\ShortHeadings{Jeongseok Kim, et al.}{Jeongseok Kim, et al.}

\mycitation{Jeongseok Kim, et al.}{Dialogue Possibilities between a Human Supervisor and UAM Air Traffic Management: Route Alteration}{Advances in Artificial Intelligence and Machine Learning.}{2023;8(6):120.}{}

\begin{document}

\title{Dialogue Possibilities between a Human Supervisor and UAM Air Traffic Management: Route Alteration}

\author{\name Jeongseok Kim \email jeongseok.kim@sk.com \\
       \addr AIX\\
       SK Telecom\\
       Seoul, 04539, Republic of Korea
       \AND
       \name Kangjin Kim \email kangjinkim@cdu.ac.kr \\
       \addr Department of Drone Systems\\
       Chodang University\\
       Jeollanam-do, 58530, Republic of Korea}

\editor{Kangjin Kim}

\maketitle

\begin{abstract}

This paper introduces a novel approach to detour management in Urban Air Traffic Management (UATM) using knowledge representation and reasoning. It aims to understand the complexities and requirements of UAM detours, enabling a method that quickly identifies safe and efficient routes in a carefully sampled environment. This method implemented in Answer Set Programming uses non-monotonic reasoning and a two-phase conversation between a human manager and the UATM system, considering factors like safety and potential impacts. The robustness and efficacy of the proposed method were validated through several queries from two simulation scenarios, contributing to the symbiosis of human knowledge and advanced AI techniques. The paper provides an introduction, citing relevant studies, problem formulation, solution, discussions, and concluding comments.

\end{abstract}


\section{Introduction}

Urban Air Mobility (UAM) has become a hot topic in the aviation industry due to technology and novel mobility options. However, the aviation industry's infrastructure is unprepared for this paradigm shift. The Korean Urban Air Mobility Concept of Operations \cite{KUAMConops10} by MOLIT describes it as a paradigm with unprecedented challenges, such as integrating low-altitude flights into dense urban environments, high-density air traffic management, and a transition to fully autonomous operations by 2035.
This mobility transition necessitates the participation of numerous stakeholders from diverse industries, resulting in a complex landscape devoid of defined data sharing mechanisms. The absence of standardization impedes UAM operations. Never before has it been more crucial to have an air traffic management system that can adapt to a heterogeneous environment and scale to accommodate growing data volume and complexity.

This investigation employs UAM Air Traffic Management (UATM) solutions to address these obstacles. We create a graph model of the UAM airway network. Each node in this concept is a "vertiport" — a vertical airport — and each connection represents a route between two adjacent vertiports. A human traffic manager supervises landing and departure operations at each vertiport and notifies the UATM system of any traffic issues.

The research presented in this paper offers an in-depth scenario illustrating the communication process of route change instructions to the relevant agents, consequently causing a modification in their currently charted routes. It is crucial to understand that the operational scope of each UATM system is not determined by its physical proximity to a vertiport, but is autonomously determined by the UATM itself. Given the constraints of communication range, this level of autonomy becomes increasingly significant. In some circumstances, a UATM might need to transmit instructions to certain agents via the UATM Network \cite{JS2022tbo}, as highlighted by our research.

This investigation aims to unravel the operational complexities inherent to the rapidly evolving UAM field. By doing so, our hope is to contribute to the development of a more robust, flexible, and scalable future urban air traffic management system. This system would be capable of accommodating a broad spectrum of stakeholders while addressing their specific data transfer needs, ultimately advancing the progress and reliability of Urban Air Mobility.

Combining with \cite{Kim2023Agent3C} and \cite{Woo2023} partly, this paper makes three contributions: first, a detour scenario; second, a theoretical enhancement to our method; and third, a newly introduced extra inquiry for an addition scenario.

We will begin with a review of the relevant literature, followed by a discussion of the problem formulation, a continuation of the primary solution, potential discussions, and a conclusion.

\section{Related Works}

The papers \cite{Reiche2018, Garrow2021, KUAMConops10, Marzouk2022} provide a comprehensive overview and challenges of UAM technology, regulatory context, potential benefits, and barriers.
The article \cite{FAAUAMConOps2} proposes a layered system approach in order to organize airspace for UAM vehicles, depending on the vehicle type, level of autonomy, and altitude.
Various ATM programs within German Aerospace Center (DLR) initiatives are discussed by the authors in \cite{Schuchardt2023}.
The authors in \cite{Kim2022} propose an assessment model for vehicle-obstacle collision hazards.
In detail, the authors of the paper \cite{PintoNeto2023} discuss cutting-edge Deep Learning techniques for ATM.


Even though previous investigations and foundations provided some solutions, including collision avoidances, it is clear that none of these works can be universally adapted due to the complexity of UATM systems, the diversity of stakeholders, the actions of massive agents, and the occurrence of unexpected aerial accidents.
In order to satisfy the needs of multiple stakeholders, we began our first step in \cite{Kim2023Agent3C} by developing a scenario for the detour of a particular corridor. 
Since non-monotonic reasoning is a type of logical reasoning concerned with the process of deriving conclusions from incomplete information, this paper employs non-monotonic reasoning to characterize route detours involving multiple UATMs.. In contrast to monotonic reasoning, in which the addition of new information to a knowledge base does not reduce the set of propositions that can be derived from that knowledge base, non-monotonic reasoning can result in the rejection of previous conclusions based on new information \cite{Reiter1988, sep-logic-nonmonotonic, sep-reasoning-defeasible}.
In the following paper \cite{Woo2023}, we showed an example scenario for changing the destination.
Considering the circumstances when a vertiport is momentarily closed, this paper illustrates the interactional procedures to reroute, expressing in Ontology relationships and rules as the predicates.

The research paper \cite{BorregoDiaz2022} is an excellent resource for the explanation of complex systems, also known as explainable AI (or XAI for short).
Their proposed epistemological model relies on knowledge representation and reasoning in particular to validate the complex system.
The article \cite{Bourguin2021} investigates a method for developing automatic classifiers capable of providing explanations based on an ontology.
This paper \cite{Ozaki2020} describes how to construct descriptive logic (DL) ontologies using five approaches based on association rule mining, formal concept analysis, inductive logic programming, computational learning theory, and neural networks. 
Despite the fact that these studies give a theoretical basis, using their ideas in our system has not proven to be a practical answer.

On the other side, the study \cite{Gebser2018} proposes an ASP-based paradigm for intra-logistics difficulties. The articles \cite{Gebser2018a, Nguyen2017} suggest using ASP to accomplish job assignment and vehicle routing for automated guided vehicles (AGVs). While they focus on demonstrating how to compute their route and provide a series of subtasks, we focus on the logic used in the process of interaction between a human manager and the system, as well as between systems.
There is a study \cite{9438207} for analyzing uncertainty in space object tracking by the Uncertainty Representation and Reasoning Evaluation Framework (URREF). URREF is an ontology that provides a common vocabulary for representing and reasoning about uncertainty. This shows how URREF can be used to model the uncertainty in the tracking process, and how this uncertainty can be used to assess the veracity, precision, and recall of the tracking results.

\section{Problem Formulation}

Let us explore a network of vertiports, with some of these vertiports being adjacent and interconnected to form a corridor, as depicted in Fig.~\ref{fig:uatm_network}. 

\begin{figure}
	\centering
	\includegraphics[width=0.85\linewidth]{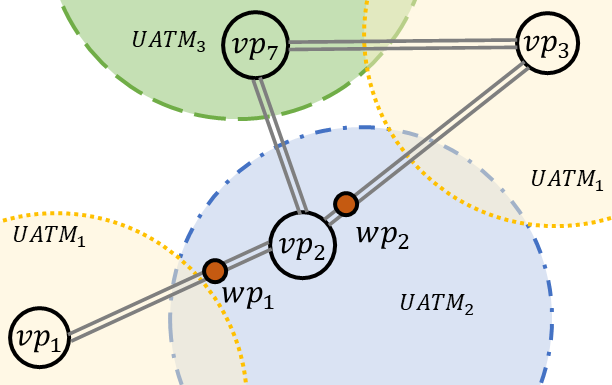}
	\caption{UATM Network, which consists of four vertiports, $vp_1$, $vp_2$, $vp_3$, and $vp_7$, bidirectionally connecting corridors between adjacent vertiports, three UATMs of $UATM_1$, $UATM_2$, and $UATM_3$, and their coverage represented with outer circles.}
	\label{fig:uatm_network}
\end{figure}

Vertiports are linked by corridors. UAM Air Traffic Management (UATM) systems manage the extensive circuits of these vertiports. Agents in this network migrate between vertiports via these corridors. UATMs can communicate directly with agents in their region. The "UATM Network" is a communication relaying mechanism that enables UATMs to exchange messages for this investigation. When the agent leaves the UATM service region, however, a direct connection cannot be established.

The itinerary for an agent's trip follows: Initially, it is positioned in a vertiport, preparing for takeoff. This agent asks permission to take off from the vertiport's traffic manager. As soon as he or she allows the agent to take the desired action, the aircraft ascends into the airspace.

Before flying, we presume that the destination is already configured and that the route has been computed.  In order arrive at the destination vertiport, the agent communicates with UATMs to share its status, including its current velocity, GPS coordinates, and other important information.  Then, UATMs are able to track the agent while taking into account the total traffic situation.

The agent requests the vertiport's traffic management to land the aircraft as it approaches. After approval, the traffic supervisor leads it to the vertipad.

This traffic system operates without human intervention. There is a human traffic manager at each vertiport. In addition to observing traffic and the vertiport environment, they interact with legacy traffic systems. Therefore, human managers increase the system's adaptability.

\subsection{Scenario 1: Reroute the Corridor}
During the process of monitoring vehicles landing and taking off for a while, the human manager in vertiport 3 ($vp_3$ for short) detected delays. It is anticipated that some collisions will occur if these delays are aggregated and transmitted to the corridors. Consequently, he or she reported this incident to the UATM (in this case, $UATM_1$). In particular, the corridor between $vp_2$ and $vp_3$ is so congested that agents are required to avoid using it. This results in agents who are en route to $vp_3$ having to detour. Alternative route consists of $vp_1$, $vp_2$, $vp_7$, and $vp_3$ vertiports.

From the perspective of $UATM_1$, this request from the manager demands extended effort. Once it gets a detour request from the manager of $vp_3$, it must locate all forthcoming agents for $vp_3$. Then, it picks $vp_2$-bound agents and delivers them a new route consisting of vertiports $vp_1$, $vp_2$, $vp_7$, and $vp_3$. Additionally, $UATM_1$ should determine if any agents are not covered. If such agents are there, they should query the entire UATMs to locate the appropriate UATMs. Suppose, for instance, that agent 3 is present at waypoint 1 ($wp_1$). Then, $UATM_2$ will reply via the network to the query. In this instance, $UATM_1$ will request that $UATM_2$ relay the new route to agent 3. After requesting that this detour message be relayed, $UATM_1$ will await $UATM_2$'s response regarding agent 3's revised plan. Once $UATM_1$ determines that its plan has been updated, it can answer $vp_3$'s manager. 
Fig.~\ref{fig:detour_scenario_1} depicts the scenario to facilitate comprehension.

\begin{figure}
	\centering
	\includegraphics[width=0.85\linewidth]{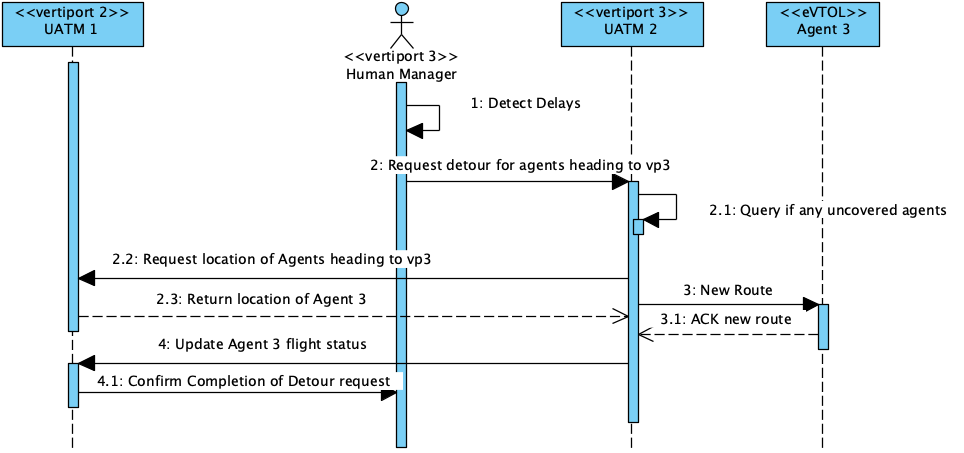}
	\caption{Re-route corridor scenario among human manager, UATMs, and agents}
	\label{fig:detour_scenario_1}
\end{figure}

\subsection{Scenario 2: Clearing the Corridor}
Now we consider a special case. While communicating with the uncovered agents and sending the detour request, a manager in $vp_2$ found that an agent was missed. That is, agent 7 was supposed to be detoured before getting into the corridor from $vp_2$ to $vp_3$. However, the $vp_2$ manager found agent 7 at waypoint 2 (or $wp_2$ for short) which is in that corridor, just before changing its route. In this case, agent 7 may be unable to reduce its velocity in time, colliding with agents ahead of it. Hence, he or she immediately reported this status to $UATM_2$. A manager in $vp_3$ requests clearing that corridor, adding to these approaching agents a round trip which is a sequence of vertiports: $vp_3$, $vp_7$, and $vp_3$.

\section{Solution}
In this section, we use Answer Set Programming to declare the problem \cite{Lifschitz1999, Baral03, Lifschitz2008}. 
ASP is comprised of a collection of logical principles that define the relationships between objects and their properties. In this context, a "answer set" is a collection of logical statements that satisfies all the program's principles and constraints \cite{gel88}. These answer sets, which represent solutions to the modeled problem, are discovered by ASP solvers \cite{GebserKKS17}.
We do this by breaking the problem down into a series of smaller questions, and then using non-monotonic reasoning to answer these questions. Once we have specified the problem, we can check if a given mission is satisfiable. If it is not, we can use the results of the non-monotonic reasoning to identify the factors that will cause the mission to fail.

\subsection{A Common Setting}
There is a shared setting for the entire solution. Due to page limitations, the Code~\ref{code:env_info}~$\sim$~\ref{code:agent_info2} are included as an appendix. This code will be executed before any other code.

\subsection{Basic Queries}
\subsubsection{Find all covered agents by $UATM_1$}
Here, we identify all agents covered by $UATM_1$. It is important to note that in Code~\ref{code:query_01}, we omitted codes for agents traveling through the step on purpose%
\footnote{These agents will be covered by the following advanced queries.}. %
 The code consists of two heads of rules. The first is \texttt{covered\_agent} and the second is \texttt{covered\_by\_uatm1}. In the second head \texttt{covered\_by\_uatm1}, we can pick the agents who are only covered by $UATM_1$. As shown in Result~\ref{res:query_01}, $UATM_1$ covers agents 1, 2, 4, and 5.

\begin{code}
	\begin{Verbatim}[breaklines,fontsize=\footnotesize]
covered_agent(A, TM) :- loc(A, T, U, V, WP), covered_wp(U, V, TM, WP).
covered_by_uatm1(A) :- covered_agent(A, 1).

#show loc/5.
#show covered_by_uatm1/1.
	\end{Verbatim}
	\caption{Find all agents that are covered by $UATM_1$}
	\label{code:query_01}
\end{code}

\begin{result}
	\begin{Verbatim}[breaklines,fontsize=\footnotesize]
$ clingo env_info.lp agent_info1.lp query01.lp
clingo version 5.6.2
Reading from env_info.lp ...
Solving...
Answer: 1
loc(3,1,1,2,19) loc(4,1,1,2,16) loc(1,1,1,2,1) loc(5,1,1,2,4) loc(6,1,1,2,2) loc(2,1,1,2,11) covered_by_uatm1(1) covered_by_uatm1(2) covered_by_uatm1(5) covered_by_uatm1(6)
SATISFIABLE

Models       : 1+
Calls        : 1
Time         : 0.000s (Solving: 0.00s 1st Model: 0.00s Unsat: 0.00s)
CPU Time     : 0.000s
	\end{Verbatim}
	\caption{All agents covered by $UATM_1$}
	\label{res:query_01}
\end{result}

\subsubsection{Change the route for covered, $vp_3$ heading agents}
The purpose of this query is to modify the route for the subsequent agents:
\begin{itemize}
	\item they are within the $UATM_1$'s coverage,
	\item their original plan was to pass the exclusive edge,
	\item their target is the $vp_3$, and
	\item they are currently on the edge between $vp_1$ and $vp_2$, which indicates that they have not yet visited the exclusive edge.
\end{itemize}

Related to this query, we presume six agents exist.
We provide the initial plan, \texttt{plan}, for these six agents, which consists of a sequence of vertiports $vp_1$, $vp_2$, and $vp_3$.
This is shown in the two lines of Code~\ref{code:agent_info1} for \texttt{plan}.
Then, we provide a new plan, \texttt{new\_plan}, which consists of the vertiports $vp_1$, $vp_2$, $vp_7$, and $vp_3$.
The first value 2 of \texttt{new\_plan} in Code~\ref{code:query_02} does not specify an agent id, but rather a step number.
It indicates that in step 2, the new plan should be adapted for a subset of agents.
Prior to entering the corridor between $vp_2$ and $vp_3$, the head \texttt{detour\_request} verifies that the covered agents' target is $vp_3$ after locating them.
Then, for the specific agent, the \texttt{detour\_request} is created for the one-time step forward.
Once one step of time has passed, agents with the \texttt{detour\_request} replace their plan with the \texttt{new\_plan}.
According to Result~\ref{res:query_02}, four agents have the detour request, and they have all altered their routes.
Since \texttt{SATISFIABLE} is returned, we know that \texttt{detour\_request} is executed.

\begin{code}
	\begin{Verbatim}[breaklines,fontsize=\footnotesize]
new_plan(2, 1, 2).
new_plan(2, 2, 7).
new_plan(2, 7, 3).

plan(A, T+1, U, V) :- plan(A, T, U, V), step(T+1), not detour_request(A, T+1).
plan(A, T+1, U1, V1) :- plan(A, T, U, V), step(T+1), new_plan(T+1, U1, V1), detour_request(A, T+1).

covered_agent(A, TM) :- loc(A, T, U, V, WP), covered_wp(U, V, TM, WP).
covered_by_uatm1(A) :- covered_agent(A, 1).

detour_request(A, T+1) :- covered_by_uatm1(A), plan(A, T, U, V), plan(A, T, 2, 3), target(A, 1, 3), edge_range(1, 2, P), loc(A, T, 1, 2, P), not step(T-1).

change_route(A, T) :- new_plan(T, U, V), plan(A, T, U, V), detour_request(A, T).
:- not change_route(A, T), new_plan(T, U, V), detour_request(A, T).

#show detour_request/2.
#show change_route/2.
#show loc/5.
	\end{Verbatim}
	\caption{Change the route for covered, $vp_3$ heading agents}
	\label{code:query_02}
\end{code}

\begin{result}
	\begin{Verbatim}[breaklines,fontsize=\footnotesize]
$ clingo env_info.lp agent_info1.lp query02.lp
clingo version 5.6.2
Reading from env_info.lp ...
Solving...
Answer: 1
loc(1,1,1,2,1) loc(3,1,1,2,3) loc(5,1,1,2,5) loc(2,1,1,2,10) loc(6,1,1,2,18) loc(4,1,1,2,19) detour_request(5,2) detour_request(3,2) detour_request(2,2) detour_request(1,2) change_route(5,2) change_route(3,2) change_route(2,2) change_route(1,2)
SATISFIABLE

Models       : 1+
Calls        : 1
Time         : 0.015s (Solving: 0.00s 1st Model: 0.00s Unsat: 0.00s)
CPU Time     : 0.000s
	\end{Verbatim}
	\caption{All covered agents' plans are renewed}
	\label{res:query_02}
\end{result}

\subsection{Advanced Queries}
\subsubsection{Find all the unreachable with $UATM_1$, but $vp_3$ heading agents}
We can search covered agents by $UATM_1$ using the rules' head \texttt{covered\_by\_uatm1}, according to Code~\ref{code:query_03}, 
We can gather the covered agents from this head.
We can identify agents that have been uncovered by comparing these agents to all agents located in the corridor between $vp_1$ and $vp_2$
The remaining agents for the head of rule \texttt{uncovered\_by\_uatm1} are listed in \ref{code:query_03}.
The Result~\ref{res:query_03} shows that three agents have been uncovered: agent 3, agent 5, and agent 6.
\begin{code}
	\begin{Verbatim}[breaklines,fontsize=\footnotesize]
covered_agent(A, TM) :- loc(A, T, U, V, WP), covered_wp(U, V, TM, WP).
uncovered_by_uatm1(A) :- not covered_agent(A, 1), loc(A, T, 1, 2, _), plan(A, T, 2, 3), target(A, T, 3).

#show loc/5.
#show uncovered_by_uatm1/1.
	\end{Verbatim}
	\caption{Find all unreachable agents by $UATM_1$}
	\label{code:query_03}
\end{code}

\begin{result}
	\begin{Verbatim}[breaklines,fontsize=\footnotesize]
$ clingo env_info.lp agent_info1.lp query03.lp
clingo version 5.6.2
Reading from env_info.lp ...
Solving...
Answer: 1
loc(1,1,1,2,1) loc(2,1,1,2,8) loc(3,1,1,2,16) loc(4,1,1,2,2) loc(5,1,1,2,19) loc(6,1,1,2,17) uncovered_by_uatm1(3) uncovered_by_uatm1(5) uncovered_by_uatm1(6)
SATISFIABLE

Models       : 1+
Calls        : 1
Time         : 0.007s (Solving: 0.00s 1st Model: 0.00s Unsat: 0.00s)
CPU Time     : 0.000s
	\end{Verbatim}
	\caption{All unreachable agents by $UATM_1$}
	\label{res:query_03}
\end{result}

\subsubsection{Change the route for all $vp_3$ heading agents}
Code~\ref{code:query_04} contains two distinct heads of the rules for \texttt{detour\_request}.
One is for agents covered by $UATM_1$ and the other is for agents unreachable by $UATM_1$.
As demonstrated in Result~\ref{res:query_04}, all covered and uncovered agents are specified, followed by their detour requests and their route modifications at time step 2.

\begin{code}
	\begin{Verbatim}[breaklines,fontsize=\footnotesize]
new_plan(2, 1, 2).
new_plan(2, 2, 7).
new_plan(2, 7, 3).

plan(A, T+1, U, V) :- plan(A, T, U, V), step(T+1), not detour_request(A, T+1).
plan(A, T+1, U1, V1) :- plan(A, T, U, V), step(T+1), new_plan(T+1, U1, V1), detour_request(A, T+1).

covered_agent(A, TM) :- loc(A, T, U, V, WP), covered_wp(U, V, TM, WP).
covered_by_uatm1(A) :- covered_agent(A, 1).
uncovered_by_uatm1(A) :- not covered_agent(A, 1), loc(A, T, 1, 2, _), plan(A, T, 2, 3), target(A, 1, 3).
covered(A, T, TM) :- loc(A, T, U, V, WP), uncovered_by_uatm1(A), covered_wp(U, V, TM, WP).

detour_request(A, T+1) :- covered_by_uatm1(A), plan(A, T, U, V), plan(A, T, 2, 3), target(A, 1, 3), edge_range(1, 2, P), loc(A, T, 1, 2, P), not step(T-1).
detour_request(A, T+1) :- covered(A, T, TM), plan(A, T, U, V), plan(A, T, 2, 3), target(A, 1, 3), edge_range(1, 2, P), loc(A, T, 1, 2, P), not step(T-1).

change_route(A, T) :- new_plan(T, U, V), plan(A, T, U, V), detour_request(A, T).
:- not change_route(A, T), new_plan(T, U, V), detour_request(A, T).

#show covered_by_uatm1/1.
#show uncovered_by_uatm1/1.
#show detour_request/2.
#show change_route/2.
#show loc/5.
	\end{Verbatim}
	\caption{Change the route for all $vp_3$ heading agents}
	\label{code:query_04}
\end{code}

\begin{result}
	\begin{Verbatim}[breaklines,fontsize=\footnotesize]
$ clingo env_info.lp agent_info1.lp query04.lp
clingo version 5.6.2
Reading from env_info.lp ...
Solving...
Answer: 1
loc(3,1,1,2,1) loc(5,1,1,2,5) loc(1,1,1,2,6) loc(2,1,1,2,9) loc(4,1,1,2,18) loc(6,1,1,2,19) covered_by_uatm1(1) covered_by_uatm1(2) covered_by_uatm1(3) covered_by_uatm1(5) detour_request(6,2) detour_request(5,2) detour_request(4,2) detour_request(3,2) detour_request(2,2) detour_request(1,2) uncovered_by_uatm1(4) uncovered_by_uatm1(6) change_route(6,2) change_route(5,2) change_route(4,2) change_route(3,2) change_route(2,2) change_route(1,2)
SATISFIABLE

Models       : 1+
Calls        : 1
Time         : 0.016s (Solving: 0.00s 1st Model: 0.00s Unsat: 0.00s)
CPU Time     : 0.000s
	\end{Verbatim}
	\caption{All $vp_3$ heading agents' plans are renewed}
	\label{res:query_04}
\end{result}

\subsubsection{Append a round detour for agents ahead of agent 7}

Code~\ref{code:query_05} gathers \texttt{ahead\_agents}. Then, it categorizes these agents to \texttt{covered\_by\_uatm2} and to \texttt{covered\_by\_other}. For each category, it sends \texttt{round\_request} for adding a round trip to the end of these agents' plan. Once the new plan is added, it can complete the \texttt{round\_route} mission.
The Result~\ref{res:query_05} shows that agent 8$\sim$12 had \texttt{round\_request}, and then finally made \texttt{round\_route}.
\begin{code}
    \begin{Verbatim}[breaklines,fontsize=\footnotesize]
new_plan(3, 3, 7).
new_plan(3, 7, 3).

ahead_agents(A, T) :- loc(A, T, U, V, WP), loc(7, T, U, V, WP2), WP > WP2.

covered_agent(A, TM) :- ahead_agents(A, T), loc(A, T, U, V, WP), covered_wp(U, V, TM, WP).
covered_by_uatm2(A) :- covered_agent(A, 2).
covered_by_other(A) :- not covered_agent(A, 2), ahead_agents(A, T), covered_agent(A, TM).

round_request(A, V, T+1) :- covered_by_uatm2(A), ahead_agents(A, T), target(A, T, V), step(T+1).
round_request(A, V, T+1) :- covered_by_other(A), ahead_agents(A, T), target(A, T, V), step(T+1).

plan(A, T+1, U, V) :- ahead_agents(A, T), plan(A, T, U, V), step(T+1).

plan(A, T, U, V) :- round_request(A, V, T), new_plan(T, U, V).
plan(A, T, V, U) :- round_request(A, V, T), new_plan(T, V, U).
target(A, T, V) :- round_request(A, V, T), plan(A, T, U, V).

round_route(A, V, T) :- round_request(A, V, T), plan(A, T, U, V), plan(A, T, V, U).
:- not round_route(A, V, T+1), ahead_agents(A, T), round_request(A, V, T+1), step(T+1).

#show covered_by_uatm2/1.
#show covered_by_other/1.
#show round_request/3.
#show round_route/3.
    \end{Verbatim}
    \caption{Append a round detour for agents ahead of agent 7}
    \label{code:query_05}
\end{code}

\begin{result}
    \begin{Verbatim}[breaklines,fontsize=\footnotesize]
$ clingo env_info.lp agent_info2.lp query05.lp
clingo version 5.6.2
Reading from env_info.lp ...
Solving...
Answer: 1
covered_by_other(9) covered_by_other(10) covered_by_other(11) covered_by_other(12) covered_by_uatm2(8) loc(7,2,2,3,2) loc(8,2,2,3,8) loc(9,2,2,3,9) loc(10,2,2,3,10) loc(11,2,2,3,11) loc(12,2,2,3,12) round_request(9,3,3) round_request(10,3,3) round_request(11,3,3) round_request(12,3,3) round_request(8,3,3) round_route(9,3,3) round_route(10,3,3) round_route(11,3,3) round_route(12,3,3) round_route(8,3,3)
SATISFIABLE

Models       : 1
Calls        : 1
Time         : 0.025s (Solving: 0.01s 1st Model: 0.00s Unsat: 0.01s)
CPU Time     : 0.000s
    \end{Verbatim}
    \caption{Plans for all agents ahead of agent 7 are renewed}
    \label{res:query_05}
\end{result}

\section{Discussion}
%
%
%
In this section, we discuss present progress, limitations, and future directions.

\subsection{Nonmonotonicity}
Through the program for the first scenario, the initial background knowledge given is that the corridor from $vp_2$ to $vp_3$ is crowded, and the focused agents are in the corridor from $vp_1$ to $vp_2$. In order to illustrate the nature of nonmonotony, agents' locations can vary. We modeled the various locations for these agents so that without considering the coverage of each UATM, they could be properly handled. Code~\ref{code:agent_info1} declares the agent's location through the choice rule for the \texttt{loc} and its following rules. These rules allow agents to be aligned evenly in different situations. With this arrangement of the agents' locations, all the corresponding rules in the queries are properly declared in order to cover general situations.

\subsection{Explainability}
The query asking that `\textbf{Change the route for all $vp_3$ heading agents}' is successful is expressed as the predicate \texttt{change\_route} and its supporting rule just followed by the predicate.
The answer demonstrates that \texttt{change\_route} is true when it is satisfiable.
In this query, the validation of the explanation is checked by the combination of the predicate, \texttt{change\_route}, which is regarded as a fact when the body of the rule is true, and the safe rule, which ensures the fact's consistency.
Assuming that all the derived rules and relationships lead to the body of the rule as true, the logical consequence makes the predicate \texttt{change\_route} by also being connected, and this justifies the answer to the query.

We observe that the explanation is somewhat abstract based on the predicates we declared.
In the second scenario, we only considered the given rules and facts, omitting the process by which the $vp_2$ could miss the agent.
In this regard, additional research must be conducted, and a more comprehensive explanation is desired.
We note that in order to validate our technique, a theoretical foundation must be addressed. Due to page constraints, we included this in the appendix.

\section{Conclusions} 
\label{sec:conclusion}

We have enumerated two UATM-related scenarios. We have implemented knowledge representation and reasoning within the framework of Answer Set Programming by employing it as an instrument for articulating system explanations. The paper then examines our current progress, prospective ambitions, and ultimate goals in this context.

\section*{Acknowledgments}
This work is supported by the Korea Agency for Infrastructure Technology Advancement(KAIA) grant funded by the Ministry of Land, Infrastructure and Transport (Grant RS-2022-00143965).

\bibliographystyle{unsrtnat}
\bibliography{references}

\begin{thebibliography}{30}
\providecommand{\natexlab}[1]{#1}
\providecommand{\url}[1]{\texttt{#1}}
\expandafter\ifx\csname urlstyle\endcsname\relax
  \providecommand{\doi}[1]{doi: #1}\else
  \providecommand{\doi}{doi: \begingroup \urlstyle{rm}\Url}\fi

\bibitem[Korea(2021)]{KUAMConops10}
UAM~Team Korea.
\newblock K-uam concept of operations, v1.0.
\newblock 2021.

\bibitem[Kim and Kim(2022)]{JS2022tbo}
Jeongseok Kim and Kangjin Kim.
\newblock Decentralized 4dt monitoring architecture for trajectory based
  operations(tbo) in the presence of multiple uatmsps.
\newblock \emph{2022 Autumn Conference of The Korean Society for Aeronautical
  and Space Sciences}, pages 1116--1118, 2022.

\bibitem[Kim and Kim(2023)]{Kim2023Agent3C}
JeongSeok Kim and Kangjin Kim.
\newblock {A}gent 3, change your route: possible conversation between a human
  manager and {UAM} {A}ir {T}raffic {M}anagement ({UATM}).
\newblock In \emph{{Robotics: Science and Systems (RSS)} workshop on Articulate
  Robots: Utilizing Language for Robot Learning}, Daegu, Korea, 2023.
\newblock \doi{https://doi.org/10.48550/arXiv.2306.14216}.
\newblock URL \url{https://arxiv.org/abs/2306.14216}.

\bibitem[Woo et~al.(2023)Woo, Kim, and Kim]{Woo2023}
Seungwan Woo, Jeongseok Kim, and Kangjin Kim.
\newblock {W}e, {V}ertiport 6, are temporarily closed: {I}nteractional
  {O}ntological {M}ethods for {C}hanging the {D}estination.
\newblock In \emph{IEEERO-MAN (RO-MAN 2023) Workshop on Ontologies for
  Autonomous Robotics (RobOntics) (Under Review)}, Busan, Korea, August 2023.
\newblock URL \url{https://arxiv.org/abs/2307.03558}.

\bibitem[Reiche et~al.(2018)Reiche, Goyal, Cohen, Serrao, Kimmel, Fernando, and
  Shaheen]{Reiche2018}
C.~Reiche, R.~Goyal, A.~Cohen, J.~Serrao, S.~Kimmel, C.~Fernando, and
  S.~Shaheen.
\newblock {U}rban {A}ir {M}obility {M}arket {S}tudy.
\newblock \emph{National Aeronautics and Space Administration (NASA)}, 2018.
\newblock \doi{http://dx.doi.org/10.7922/G2ZS2TRG}.
\newblock URL \url{https://escholarship.org/uc/item/0fz0x1s2}.

\bibitem[Garrow et~al.(2021)Garrow, German, and Leonard]{Garrow2021}
Laurie~A. Garrow, Brian~J. German, and Caroline~E. Leonard.
\newblock Urban air mobility: A comprehensive review and comparative analysis
  with autonomous and electric ground transportation for informing future
  research.
\newblock \emph{Transportation Research Part C: Emerging Technologies},
  132:\penalty0 103377, 2021.
\newblock ISSN 0968-090X.
\newblock \doi{https://doi.org/10.1016/j.trc.2021.103377}.
\newblock URL
  \url{https://www.sciencedirect.com/science/article/pii/S0968090X21003788}.

\bibitem[Marzouk(2022)]{Marzouk2022}
Osama~A. Marzouk.
\newblock Urban air mobility and flying cars: Overview, examples, prospects,
  drawbacks, and solutions.
\newblock \emph{Open Engineering}, 12\penalty0 (1):\penalty0 662--679, 2022.
\newblock \doi{doi:10.1515/eng-2022-0379}.
\newblock URL \url{https://doi.org/10.1515/eng-2022-0379}.

\bibitem[Administration(2023)]{FAAUAMConOps2}
Federal~Aviation Administration.
\newblock Faa's urban air mobility (uam) concept of operations version 2.0,
  2023.
\newblock URL
  \url{https://www.faa.gov/sites/faa.gov/files/Urban%20Air%20Mobility%20%28UAM%29%20Concept%20of%20Operations%202.0_0.pdf}.

\bibitem[Schuchardt et~al.(2023)Schuchardt, Geister, Lüken, Knabe, Metz,
  Peinecke, and Schweiger]{Schuchardt2023}
Bianca~I. Schuchardt, Dagi Geister, Thomas Lüken, Franz Knabe, Isabel~C. Metz,
  Niklas Peinecke, and Karolin Schweiger.
\newblock Air traffic management as a vital part of urban air mobility -- a
  review of dlr's research work from 1995 to 2022.
\newblock \emph{Aerospace}, 10\penalty0 (1), 2023.
\newblock ISSN 2226-4310.
\newblock \doi{10.3390/aerospace10010081}.
\newblock URL \url{https://www.mdpi.com/2226-4310/10/1/81}.

\bibitem[Kim and Lee(2022)]{Kim2022}
Dongsin Kim and Keumjin Lee.
\newblock Surveillance-based risk assessment model between urban air mobility
  and obstacles.
\newblock \emph{Journal of the Korean Society for Aviation and Aeronautics},
  30\penalty0 (3):\penalty0 19--27, 2022.
\newblock \doi{10.12985/ksaa.2022.30.3.019}.
\newblock URL \url{https://doi.org/10.12985/ksaa.2022.30.3.019}.

\bibitem[Pinto~Neto et~al.(2023)Pinto~Neto, Baum, Almeida, Camargo, and
  Cugnasca]{PintoNeto2023}
Euclides~Carlos Pinto~Neto, Derick~Moreira Baum, Jorge Rady~de Almeida,
  João~Batista Camargo, and Paulo~Sergio Cugnasca.
\newblock Deep learning in air traffic management (atm): A survey on
  applications, opportunities, and open challenges.
\newblock \emph{Aerospace}, 10\penalty0 (4), 2023.
\newblock ISSN 2226-4310.
\newblock \doi{10.3390/aerospace10040358}.
\newblock URL \url{https://www.mdpi.com/2226-4310/10/4/358}.

\bibitem[Reiter(1988)]{Reiter1988}
R.~Reiter.
\newblock \emph{Nonmonotonic Reasoning}, page 439–481.
\newblock Morgan Kaufmann Publishers Inc., San Francisco, CA, USA, 1988.
\newblock ISBN 0934613672.

\bibitem[Strasser and Antonelli(2019)]{sep-logic-nonmonotonic}
Christian Strasser and G.~Aldo Antonelli.
\newblock {Non-monotonic Logic}.
\newblock In Edward~N. Zalta, editor, \emph{The {Stanford} Encyclopedia of
  Philosophy}. Metaphysics Research Lab, Stanford University, {S}ummer 2019
  edition, 2019.

\bibitem[Koons(2022)]{sep-reasoning-defeasible}
Robert Koons.
\newblock {Defeasible Reasoning}.
\newblock In Edward~N. Zalta, editor, \emph{The {Stanford} Encyclopedia of
  Philosophy}. Metaphysics Research Lab, Stanford University, {S}ummer 2022
  edition, 2022.

\bibitem[Borrego-Díaz and Galán-Páez(2022)]{BorregoDiaz2022}
Joaquín Borrego-Díaz and Juan Galán-Páez.
\newblock Knowledge representation for explainable artificial intelligence:
  Modeling foundations from complex systems.
\newblock \emph{Complex \& Intelligent Systems}, 8, 01 2022.
\newblock \doi{10.1007/s40747-021-00613-5}.

\bibitem[Bourguin et~al.(2021)Bourguin, Lewandowski, Bouneffa, and
  Ahmad]{Bourguin2021}
Gr{\'e}gory Bourguin, Arnaud Lewandowski, Mourad Bouneffa, and Adeel Ahmad.
\newblock Towards ontologically explainable classifiers.
\newblock In Igor Farka{\v{s}}, Paolo Masulli, Sebastian Otte, and Stefan
  Wermter, editors, \emph{Artificial Neural Networks and Machine Learning --
  ICANN 2021}, pages 472--484, Cham, 2021. Springer International Publishing.
\newblock ISBN 978-3-030-86340-1.

\bibitem[Ozaki(2020)]{Ozaki2020}
A.~Ozaki.
\newblock Learning description logic ontologies: Five approaches. where do they
  stand?
\newblock \emph{KI - K{\"u}nstliche Intelligenz}, pages 1--11, 2020.

\bibitem[Gebser et~al.(2018{\natexlab{a}})Gebser, Obermeier, Otto, Schaub,
  Sabuncu, Nguyen, and Son]{Gebser2018}
Martin Gebser, Philipp Obermeier, Thomas Otto, Torsten Schaub, Orkunt Sabuncu,
  Van Nguyen, and Tran~Cao Son.
\newblock Experimenting with robotic intra-logistics domains.
\newblock \emph{{TPLP}}, 18\penalty0 (3-4):\penalty0 502--519,
  2018{\natexlab{a}}.

\bibitem[Gebser et~al.(2018{\natexlab{b}})Gebser, Obermeier, Schaub,
  Ratsch{-}Heitmann, and Runge]{Gebser2018a}
Martin Gebser, Philipp Obermeier, Torsten Schaub, Michel Ratsch{-}Heitmann, and
  Mario Runge.
\newblock Routing driverless transport vehicles in car assembly with answer set
  programming.
\newblock \emph{{TPLP}}, 18\penalty0 (3-4):\penalty0 520--534,
  2018{\natexlab{b}}.

\bibitem[Nguyen et~al.(2017)Nguyen, Obermeier, Son, Schaub, and
  Yeoh]{Nguyen2017}
Van Nguyen, Philipp Obermeier, Tran~Cao Son, Torsten Schaub, and William Yeoh.
\newblock Generalized target assignment and path finding using answer set
  programming.
\newblock In \emph{{IJCAI}}, pages 1216--1223. ijcai.org, 2017.

\bibitem[Blasch et~al.(2021)Blasch, Shen, Chen, Sheaff, and Pham]{9438207}
Erik Blasch, Dan Shen, Genshe Chen, Carolyn Sheaff, and Khanh Pham.
\newblock Space object tracking uncertainty analysis with the urref ontology.
\newblock In \emph{2021 IEEE Aerospace Conference (50100)}, pages 1--9, 2021.
\newblock \doi{10.1109/AERO50100.2021.9438207}.

\bibitem[Lifschitz(1999)]{Lifschitz1999}
Vladimir Lifschitz.
\newblock Answer set planning.
\newblock In Michael Gelfond, Nicola Leone, and Gerald Pfeifer, editors,
  \emph{Logic Programming and Nonmonotonic Reasoning}, pages 373--374, Berlin,
  Heidelberg, 1999. Springer Berlin Heidelberg.
\newblock ISBN 978-3-540-46767-0.

\bibitem[Baral(2003)]{Baral03}
Chitta Baral.
\newblock \emph{{Knowledge Representation, Reasoning and Declarative Problem
  Solving}}.
\newblock Cambridge University Press, February 2003.
\newblock ISBN 0521818028.
\newblock URL
  \url{http://www.amazon.com/exec/obidos/redirect?tag=citeulike07-20\&path=ASIN/0521818028}.

\bibitem[Lifschitz(2008)]{Lifschitz2008}
Vladimir Lifschitz.
\newblock What is answer set programming?
\newblock In \emph{Proceedings of the 23rd National Conference on Artificial
  Intelligence - Volume 3}, AAAI'08, page 1594–1597, Chicago, Illinois, 2008.
  AAAI Press.
\newblock ISBN 9781577353683.

\bibitem[Gelfond and Lifschitz(1988)]{gel88}
Michael Gelfond and Vladimir Lifschitz.
\newblock The stable model semantics for logic programming.
\newblock In Robert Kowalski, Bowen, and Kenneth, editors, \emph{Proceedings of
  International Logic Programming Conference and Symposium}, pages 1070--1080.
  MIT Press, 1988.
\newblock URL \url{http://www.cs.utexas.edu/users/ai-lab?gel88}.

\bibitem[Gebser et~al.(2017)Gebser, Kaminski, Kaufmann, and
  Schaub]{GebserKKS17}
Martin Gebser, Roland Kaminski, Benjamin Kaufmann, and Torsten Schaub.
\newblock Multi-shot {ASP} solving with clingo.
\newblock \emph{CoRR}, abs/1705.09811, 2017.

\bibitem[Davis and Putnam(1960)]{Davis1960}
Martin Davis and Hilary Putnam.
\newblock A computing procedure for quantification theory.
\newblock \emph{J. ACM}, 7\penalty0 (3):\penalty0 201–215, jul 1960.
\newblock ISSN 0004-5411.
\newblock \doi{10.1145/321033.321034}.
\newblock URL \url{https://doi.org/10.1145/321033.321034}.

\bibitem[Davis et~al.(1962)Davis, Logemann, and Loveland]{Davis1962rohtua}
Martin Davis, George Logemann, and Donald Loveland.
\newblock A machine program for theorem-proving.
\newblock \emph{Commun. ACM}, 5\penalty0 (7):\penalty0 394–397, jul 1962.
\newblock ISSN 0001-0782.
\newblock \doi{10.1145/368273.368557}.
\newblock URL \url{https://doi.org/10.1145/368273.368557}.

\bibitem[Bayardo and Schrag(1998)]{Bayardo1998}
Roberto Bayardo and Robert Schrag.
\newblock Using csp look-back techniques to solve real-world sat instances.
\newblock \emph{Proceedings of the National Conference on Artificial
  Intelligence}, 03 1998.

\bibitem[Gomes et~al.(2000)Gomes, Selman, Crato, and
  Kautz]{Gomes2000HeavyTailedPI}
Carla~Pedro Gomes, Bart Selman, Nuno Crato, and Henry~A. Kautz.
\newblock Heavy-tailed phenomena in satisfiability and constraint satisfaction
  problems.
\newblock \emph{Journal of Automated Reasoning}, 24:\penalty0 67--100, 2000.
\newblock URL \url{https://api.semanticscholar.org/CorpusID:1748869}.

\end{thebibliography}

\begin{appendix}
\section{Common Settings}
\subsection{Information about the Environment}
\begin{code}
	\begin{Verbatim}[breaklines,fontsize=\footnotesize]
% common settings
%    initial information for environment

uatm(1..3). agent(1..20). vp(1..7).

% edge(VP_u, VP_v) :
%    there is a corridor from VP_u to VP_v
edge(1, 2). edge(2, 3). edge(2, 7). edge(7, 3).

% cover(UATM_i, VP_u) :
%    UATM_i coveres VP_u
cover(1, 1). cover(1, 3).
cover(2, 2).
cover(3, 7).

% edge_range(VP_i, VP_j, P) :
%    corridor from VP_i to VP_j has range P.
edge_range(1, 2, 1..20).
edge_range(2, 3, 1..13).
edge_range(2, 7, 1..22).

% covered_wp(VP_u, VP_v, UATM_i, P) :
%    UATM_i covers a corridor from VP_u to VP _v within the range P.
covered_wp(1, 2, 1, P) :- edge_range(1, 2, P), P < 16.
covered_wp(1, 2, 2, P) :- edge_range(1, 2, P), 7 <= P.
covered_wp(2, 3, 1, P) :- edge_range(2, 3, P), 9 <= P.
covered_wp(2, 3, 2, P) :- edge_range(2, 3, P), P < 9.
covered_wp(2, 7, 2, P) :- edge_range(2, 7, P), P < 8.
covered_wp(2, 7, 3, P) :- edge_range(2, 7, P), 20 <= P.

step(1..3).
	\end{Verbatim}
	\caption{Information about the Environment}
	\label{code:env_info}
\end{code}
\pagebreak
\subsection{Information about the Agents for the first four queries}
\begin{code}
	\begin{Verbatim}[breaklines,fontsize=\footnotesize]
% common settings :
%    initial information for agents

% loc(AGENT_a, STEP_t, VP_i, VP_j, WP_p) :
%    At STEP_t, AGENT_a locates at WP_p on the corridor between VP_i and VP_j.
1{loc(A, 1, 1, 2, WP): edge_range(1, 2, WP)}1 :- agent(A), A <= 6.
:- loc(A1, 1, 1, 2, WP), loc(A2, 1, 1, 2, WP), A1 != A2.

% uatm1_wps(WP_p) :
%    WP_p is the waypoint in the corridor from vp1 to vp2, and uatm1 covers the waypoint.
uatm1_wps(WP) :- covered_wp(1, 2, 1, WP1), covered_wp(1, 2, 2, WP2), edge_range(1, 2, WP), WP != WP2, WP == WP1.

% uatm2_wps(WP_p) :
%    WP_p is the waypoint in the corridor from vp1 to vp2, and uatm2 covers the waypoint.
uatm2_wps(WP) :- covered_wp(1, 2, 1, WP1), covered_wp(1, 2, 2, WP2), edge_range(1, 2, WP), WP != WP1, WP == WP2.

% uatm1_2_both(WP_p) :
%    WP_p is the waypoint in the corridor from vp1 to vp2, and both uatm1 and uatm2 cover the waypoint.
uatm1_2_both(WP) :- covered_wp(1, 2, 1, WP1), covered_wp(1, 2, 2, WP2), edge_range(1, 2, WP), WP == WP1, WP == WP2.

% u1_only(N):
%    N is the number of waypoints in uatm1_wps(WP).
u1_only(N) :- N = #count{A:uatm1_wps(WP), not uatm2_wps(WP), loc(A, 1, 1, 2, WP), agent(A)}.

% u2_only(N):
%    N is the number of waypoints in uatm2_wps(WP).
u2_only(N) :- N = #count{A:uatm2_wps(WP), not uatm1_wps(WP), loc(A, 1, 1, 2, WP), agent(A)}.

% u1_2_both(N):
%    N is the number of waypoints in both uatm1_wps(WP) and uatm2_wps(WP).
u1_2_both(N) :- N = #count{A:uatm1_wps(WP), uatm2_wps(WP), loc(A, 1, 1, 2, WP), agent(A)}.

% focused_agent_number(N):
%    N is the number of agents that locate at the corridor from vp1 to vp2.
focused_agent_number(N) :- N = #count{A: loc(A, 1, 1, 2, WP), agent(A), edge_range(1, 2, WP)}.
:- u1_only(N), focused_agent_number(N).
:- u2_only(N), focused_agent_number(N).
:- u1_2_both(N), focused_agent_number(N).
:- u1_2_both(N), N == 0.
:- u1_only(N), N = 0.
:- u2_only(N), N = 0.

% plan(AGENT_a, STEP_t, VP_i, VP_j) :
%    At STEP_t, AGENT_a has a part of plan to move from VP_i to VP_j.
plan(A, 1, 1, 2) :- agent(A), 1 <= A, A <= 6.
plan(A, 1, 2, 3) :- agent(A), 1 <= A, A <= 6.

% we assume that every plan is acyclic.
source(A, 1, U) :- agent(A), plan(A, 1, U, V), not plan(A, 1, _, U).
target(A, 1, V) :- agent(A), plan(A, 1, U, V), not plan(A, 1, V, _).
	\end{Verbatim}
	\caption{Information about the Agents for the first four queries}
	\label{code:agent_info1}
\end{code}
\pagebreak
\subsection{Information about the Agents for the last query}
\begin{code}
	\begin{Verbatim}[breaklines,fontsize=\footnotesize]
% common settings :
%    initial information for a missed agent and agents that locate at the condensed corridor

% loc(AGENT_a, STEP_t, VP_i, VP_j, WP_p) :
%    At STEP_t, AGENT_a locates at WP_p on the corridor between VP_i and VP_j.
loc(7, 2, 2, 3, 2).
loc(8, 2, 2, 3, 8).
loc(9, 2, 2, 3, 9).
loc(10, 2, 2, 3, 10).
loc(11, 2, 2, 3, 11).
loc(12, 2, 2, 3, 12).

% plan(AGENT_a, STEP_t, VP_i, VP_j) :
%    At STEP_t, AGENT_a has a part of plan to move from VP_i to VP_j.
plan(A, 2, 1, 2) :- agent(A), 7 <= A, A <= 12.
plan(A, 2, 2, 3) :- agent(A), 7 <= A, A <= 12.

source(A, 2, U) :- agent(A), plan(A, 2, U, V), not plan(A, 2, _, U).
target(A, 2, V) :- agent(A), plan(A, 2, U, V), not plan(A, 2, V, _).
	\end{Verbatim}
	\caption{Information about the Agents for the last query}
	\label{code:agent_info2}
\end{code}

\section{Validation Analysis}
\subsection{Validating the answer using graph reachability}
In this section, we provide a theoretical foundation for our approach. Particularly, we prove that the answer returned from ASP is valid using the graph reachability technique.

\begin{prop}
	Every answer returned from ASP is valid.
\end{prop}

\begin{proof}
	Without loss of generality, define $G$, $S$, and $T$ as follows:
	\begin{itemize}
		\item Let $G$ be the graph of reachable states in ASP,
		\item Let $S$ be the set of states that satisfy the goal, and
		\item Let $T$ be the set of states that are returned by the ASP solver.
	\end{itemize}
	We want to show that $T \subset S$.
	\begin{itemize}
		\item For any state $s \in T$, there exists a path $p$ from the initial state to $s \in G$.
		\item Since $s \in T$, it must satisfy the goal.
		\item Therefore, if any state $s$ that is reachable from the initial state in $G$ and satisfies the goal, then $s \in T$. 
	\end{itemize}
	This shows that $T \subset S$. In other words, all of the states that are returned by the ASP solver are valid solutions to the problem.
\end{proof}

We remark that our desired solutions are in a subset of $T$.
In addition, it is worth noting that this technique is similar to proving the validation in model checking. Both involve the construction of a graph of all possible states that the system can reach and then searching the graph for a state that satisfies the goal. Despite the fact that we used graph reachability analysis, the computational technique used in the building of many answer set solvers is an improvement to the DPLL algorithm \cite{Davis1960, Davis1962rohtua, Bayardo1998, Gomes2000HeavyTailedPI}.

\end{appendix}

\end{document}